\documentclass[sn-basic]{sn-jnl}
\usepackage{manyfoot}%
\usepackage{fullpage}
\usepackage{algorithm}%
\usepackage{algorithmicx}%
\usepackage{algpseudocode}%

\floatname{algorithm}{Algorithm}

\usepackage{graphicx}
\usepackage{amsmath}
\usepackage{psfrag}
\usepackage{color,amssymb}
\usepackage{amsthm}
\theoremstyle{plain}%
\newtheorem{theorem}{Theorem}%  meant for continuous numbers
\newtheorem{definition}{Definition}%
\usepackage{thmtools, thm-restate}

\begin{document}

\title[Article Title]{Encoder Embedding for General Graph and Node Classification}

\author*[1]{\fnm{Cencheng} \sur{Shen}}\email{shenc@udel.edu}

\affil[1]{\orgdiv{Department of Applied Economics and Statistics}, \orgname{University of Delaware}, \orgaddress{\city{Newark}, \postcode{19716}, \state{DE}, \country{US}}}

\abstract{Graph encoder embedding, a recent technique for graph data, offers speed and scalability in producing vertex-level representations from binary graphs. In this paper, we extend the applicability of this method to a general graph model, which includes weighted graphs, distance matrices, and kernel matrices. We prove that the encoder embedding satisfies the law of large numbers and the central limit theorem on a per-observation basis. Under certain condition, it achieves asymptotic normality on a per-class basis, enabling optimal classification through discriminant analysis. These theoretical findings are validated through a series of experiments involving weighted graphs, as well as text and image data transformed into general graph representations using appropriate distance metrics. }


%%%%%%%%%%%%%%%%%%%%%%%%%%%%%%%%%%%%%%%%%%%%%%%%%%%%%%%%%%%%%%%%%%%%%%%%%%%%%%
\maketitle  
\section{Introduction}
Graph data have gained increasing significance in various real-world scenarios. Traditionally, a binary graph is represented by an $n \times n$ adjacency matrix $\mathbf{A}$, where $\mathbf{A}(i,j)=1$ indicates the existence of an edge between vertices $i$ and $j$, and $\mathbf{A}(i,j)=0$ indicates the absence of an edge. This binary graph structure is prevalent in real-world contexts, including social networks, communication networks, webpage hyperlinks, and biological systems \citep{GirvanNewman2002, newman2003structure, barabasi2004network, boccaletti2006complex, VarchneyEtAl2011, ugander2011anatomy}.

To explore graph data, graph embedding is a fundamental and versatile approach, including various techniques like spectral embedding \citep{RoheEtAl2011, SussmanEtAl2012}, graph convolutional neural networks \citep{kipf2017semi, Wu2019ACS}, and node2vec \citep{grover2016node2vec, node2vec2021}, among others. Graph embedding yields low-dimensional representations that preserve structural graph information, facilitating diverse downstream inference tasks like community detection \citep{KarrerNewman2011, ZhaoLevinaZhu2012}, vertex classification \citep{perozzi2014deepwalk, kipf2017semi}, outlier detection \citep{Ranshous2015, akoglu2015graph}, and more.

When vertex labels are available, a recent approach known as one-hot graph encoder embedding has demonstrated excellent performance in terms of speed, scalability, and vertex classification \citep{GEE1}, making it an ideal choice for handling large graph data compared to other graph embedding techniques. It has also shown effectiveness for multiple-graph and dynamic-graph inference \citep{GEEDynamics, GEEFusion}, graph correlation and hypothesis testing \citep{GEECorr}, and latent community recovery \citep{GEERefine}. However, while there are indications that the method may also be applicable to weighted graphs or distance matrices, its theoretical properties were primarily established for binary graphs.

In this manuscript, we delve into the asymptotic properties of one-hot encoder embedding for a general graph model, and investigate the corresponding numerical performance in supervised learning. We begin with a general graph model that encompasses various pairwise functions, allowing for the representation of binary graphs, weighted graphs, distance matrices, inner products, and kernel matrices. Building upon this general graph model, we define the resulting encoder embedding and establish the law of large numbers and central limit theorems for the per-vertex embedding. Furthermore, we demonstrate that when the per-class embedding shares the same normal distribution, discriminant analysis is approximately the Bayes optimal classifier for the encoder embedding.

Our theoretical findings are validated by a variety of simulations and real data experiments. We demonstrate asymptotic normality and optimal classification performance using ground-truth simulated data. Additionally, we showcase the method's excellent performance on real data, including weighted graphs and general graphs from distance transformations of text and image data. All experiments were conducted on a local PC equipped with MATLAB 2024a, running on Windows 10 with a 16-core Intel CPU and 64GB of memory.
   
\section{General Graph Model and Encoder Transformation}
\subsection{Model Definition}
Given $(X, Y) \in \mathbb{R}^{p} \times [K]$, where $Y$ follows a categorical distribution with prior probabilities $\{\pi_k \in (0,1], \sum_{k=1}^{K} \pi_k = 1\}$, and $X$ is a latent variable following a K-component mixture distribution:
\begin{align*}
X \sim \sum_{k=1}^{K} \pi_k f_{X|Y=k}(x), 
\end{align*}

Assume an independent copy of $(X, Y)$ denoted by $(U, V)$, as well as $m$ additional pairs of random variables:
\begin{align*}
(U_j, V_j) \stackrel{i.i.d.}{\sim} F_{UV}, j=1,\ldots,m,
\end{align*}
where i.i.d. means independently and identically distributed. Let
\begin{align*}
\mathcal{\vec{U}} &= [U_1; U_2; \cdots; U_m] \in \mathbb{R}^{m \times p}\\
\mathcal{\vec{V}} &= [V_1, V_2, \cdots, V_m] \in [K]^{m}
\end{align*}
denote the resulting random matrices. Note that all vectors such as $X$ and $U_j$ are assumed to be row vectors in the paper.

Next, given a pairwise function $\kappa(\cdot,\cdot): \mathbb{R}^{p} \times \mathbb{R}^{p} \rightarrow \mathbb{R}$, we define the general graph variable $A$ as follows. We say $A$ follows the graph distribution:
\begin{align*}
A \sim \text{Graph}(m,X, \kappa) \in \mathbb{R}^{m},
\end{align*}
if and only if each dimension $A_j$ of $A$ satisfies:
\begin{align*}
A_j = \kappa(X, U_j)
\end{align*}
for $j=1,\ldots,m$. For example, considering an $n \times n$ adjacency matrix $\mathbf{A}$, we can perceive each row of $\mathbf{A}$ (excluding diagonals) as identically (though not necessarily independently) distributed as $\text{Graph}(n-1,X, \kappa)$.

Traditionally, the graph variable $A$ is observed while the original latent random variable $X$ is not. In other cases, the graph variable $A$ may be transformed from an observed random variable $X$. In a supervised classification setting, $A$ represents the testing data, $Y$ is the true but unknown label, and $\mathcal{\vec{V}}$ are the observed training labels. In a classification framework, assume $g(\cdot): \mathbb{R}^{m} \rightarrow \{1,2,\ldots,K\}$ is the classifier, then the probability error is defined by:
\begin{align*}
L = \text{Prob}(g(A) \neq Y).
\end{align*}

\subsection{Examples}
While the graph variable $A$ is binary for binary graphs, a weighted graph can have any non-negative values in each dimension of $A$. In general, the entries of $A$ can take on any values depending on how $\kappa(\cdot,\cdot)$ is defined. Examples include:

\begin{itemize}
\item Euclidean Distance: 
\begin{align*}
\kappa(x,u)= \|x-u\|_2,
\end{align*}
\item Inner Product:
\begin{align*}
\kappa(x,u)= xu^{T},
\end{align*}
\item Normalized Inner Product:
\begin{align*}
\kappa(x,u)= \frac{xu^{T}}{\|x\| \|u\|},
\end{align*}
\item Radial basis function kernel:
\begin{align*}
\kappa(x,u)= exp(-\frac{\|x-u\|^2}{2\sigma^2}),
\end{align*}
\item Random Dot Product Graph \citep{YoungScheinerman2007,JMLR:v18:17-448}: 
\begin{align*}
\kappa(x,u)= \operatorname{Bernoulli}(x u^{T}).
\end{align*}
\item Stochastic Block Model \citep{HollandEtAl1983, SnijdersNowicki1997}: Let $B=[B(k,l)] \in [0,1]^{K \times K}$ be a block probability matrix. Then
\begin{align*}
\kappa(x,u)= \operatorname{Bernoulli}(B(y, v)),
\end{align*}
where $y$ and $v$ are the underlying labels of $x$ and $u$, respectively.
\item Degree-Corrected Stochastic Block Model \citep{ZhaoLevinaZhu2012}: Let $B=[B(k,l)] \in [0,1]^{K \times K}$ be a block probability matrix, and $\theta_x$ and $\theta_u$ be two degree parameters for $x$ and $u$ respectively. Then
\begin{align*}
\kappa(x,u)= \operatorname{Bernoulli}(\theta_x \theta_u B(y, v)).
\end{align*}
\end{itemize}
Thus, the general graph variable $A$ has the capacity to capture all possible pairwise relationships, contingent on the definition of $\kappa(\cdot,\cdot)$.

\subsection{The Encoder Transformation}
\begin{definition}
Given an $m$-dimensional graph variable $A$, we define the encoder transformation $Z=h(A)$ as follows: for each $k=1,\ldots,K$, compute
\begin{align*}
m_k=\sum\limits_{j=1}^{m} 1(V_j=k),
\end{align*}
where $1(V_j=k)$ equals $1$ if $V_j=k$ and $0$ otherwise. Then compute
\begin{align}
\label{eq1}
Z_k = \sum\limits_{j=1}^{m} A_j 1(V_j=k) / m_k \in \mathbb{R}.
\end{align}
The resulting $K$-dimensional embedding $Z = [Z_1, Z_2, \ldots, Z_K] \in \mathbb{R}^{K}$ is referred to as the encoder embedding.
\end{definition}

\subsection{Sample Method}
Given an arbitrary $n \times n$ sample graph $\mathbf{A}$ and the corresponding label vector $\mathbf{Y}$, the following steps compute the sample encoder embedding and conduct the classification task.
\begin{itemize}
\item \textbf{Input}: A general graph matrix $\mathbf{A} \in \mathbb{R}^{n \times n}$ and a label vector $\mathbf{Y} \in \{0,1,\ldots,K\}^{n}$, where values from $1$ to $K$ indicate known labels, and the category $0$ is used to represent observations with unknown labels.
\item  \textbf{Step 1}: Calculate the number of known observations per class, denoted as:
\begin{align*}
n_k = \sum_{i=1}^{n} 1(Y_i=k)
\end{align*}
for $k=1,\ldots,K$. 
\item \textbf{Step 2}: Construct the normalized one-hot matrix $\mathbf{W} \in [0,1]^{n \times K}$ as follow: for each observation $i=1,\ldots,n$, set
\begin{align*}
\mathbf{W}(i, k) = 1 / n_k
\end{align*} 
if and only if $\mathbf{Y}_i=k$, and $0$ otherwise. Note that vertices with unknown labels are effectively assigned zero values, i.e., $\mathbf{W}(i, :)$ is a zero vector if $\mathbf{Y}_i=0$.
\item \textbf{Step 3}: Perform the following matrix multiplication:
\begin{align*}
\mathbf{Z}=\mathbf{A} \mathbf{W} \in \mathbb{R}^{n \times K}.
\end{align*}
%and then update each row as follows:
%\begin{align*}
%\mathbf{Z}(i,\mathbf{Y}_i)=\mathbf{Z}(i,\mathbf{Y}_i) * \frac{n_{\mathbf{Y}_i}}{n_{\mathbf{Y}_i}-1}
%\end{align*}
%for each $i=1,\ldots,n$.
\item \textbf{Step 4 (Normalization)}: Given $\mathbf{Z}$ from step 3, for each $i$ where $\|\mathbf{Z}(i, \cdot)\| > 0$, compute a normalized embedding as follows:
\begin{align*}
\mathbf{Z}(i, \cdot) = \frac{\mathbf{Z}(i, \cdot)}{\|\mathbf{Z}(i, \cdot)\|}.
\end{align*}
\item \textbf{Step 5 (Classification Evaluation)}:Let $trn$ be the index set of the training vertices, and $tsn$ be the index set of the testing vertices. Build the linear discriminant analysis model $g_{trn}(\cdot)$ based on the training data $(\mathbf{Z}(trn, :), \mathbf{Y}_{trn})$. Then compute the empirical classification error by 
\begin{align*}
\hat{L} =  \frac{\sum_{i \in tsn} 1(g_{trn}(\mathbf{Z}(i, :)) \neq \mathbf{Y}_{i})}{|tsn|}
\end{align*}
\item \textbf{Output}: The encoder embedding $\mathbf{Z}$ and the classification error $\hat{L}$. 
\end{itemize}

To maintain simplicity, we present a straightforward hold-out evaluation in Step 5. However, it can be readily substituted with a 5-fold or 10-fold evaluation, as employed in the experimental section.

It is worth noting that the sample method involves an additional normalization step in Step 4, which is not present in the population version. Data normalization is a common technique in many machine learning methods. In this case, normalization helps align the embedding of anomaly data with other data points, especially useful when dealing with sparse graph data \citep{GEEClustering} or entries with very large weights. This enhances the numerical estimation in subsequent discriminant analysis.

In the case of a dense graph, both the time complexity and storage requirements are $O(n^2)$. The $n^2$ part involves only a single matrix multiplication, making it significantly faster than any competing graph embedding approach. Furthermore, for sparse graphs, it can be accelerated to $O(s+nK)$, as demonstrated in \cite{GEE1, GEEDynamics}.

\section{Asymptotic Theorems}

Here we present the asymptotic theorems. We define $\vec{m}=[m_1,m_2,\ldots,m_k]$ as the vector representing the sample size per class. We use $Diag(\cdot)$ to denote the vector-to-diagonal-matrix transformation, and $\circ$ to represent the entry-wise product. 

\subsection{Assumptions}

In all the theorem proofs, the following assumptions are consistently maintained:
\begin{itemize}
    \item $m=\Theta(n)$ and $m_k =\Theta(n)$;
    \item At any $n$, $\mathcal{\vec{V}}$ is a fixed sample vector.
    \item $A_j$ has finite moments for each $j=1,\ldots,m$.
\end{itemize}
The first assumption means that the number of known labels $m$ and the known labels per-class $m_k$ increase proportionally with the sample size $n$. This assumption aligns with the nature of supervised learning, where $m$ represents the number of training labels, and $m_k$ corresponds to the number of training observations in class $k$. In a supervised setting, both of these quantities naturally increase with the sample size.

The second assumption implies that class labels are considered fixed during the proofs, or, equivalently, the proofs are presented with conditioning on the class labels. As a result, $m_k$ remains fixed and is not treated as a random variable at any sample size. This assumption simplifies the proof process. However, it's important to note that all the theorems remain valid even without conditioning on $\mathcal{\vec{V}}$, as the asymptotic results hold independently of the actual class labels.

The third assumption ensures that the graph weights are well-behaved, a common condition that holds when the raw data have finite moments or when the function $\kappa(\cdot,\cdot)$ is properly normalized.

\subsection{Asymptotic Normality}

We first consider the per-vertex embedding when conditioning on $X=x$. 

\begin{theorem}
\label{thm1}
As $n$ increases to infinity, the encoder embedding conditioned on $X=x$ satisfies the weak law of large number and central limit theorem:
\begin{align*}
&Z|(X=x)  \stackrel{prob}{\rightarrow} \mu_{x} \\
&Diag(\vec{m})^{0.5} \circ (Z|(X=x) - \mu_{x}) \stackrel{dist}{\rightarrow} \mathcal{N}(0, \Sigma_{x}).
\end{align*}
Here, $\mu_{x} \in \mathbb{R}^{K}$ is a conditional mean vector where each dimension satisfies
\begin{align*}
\mu_{x}(k)= E(\kappa(x,U)|V=k)
\end{align*}
for $k=1,\ldots,K$, and $\Sigma_{x} \in \mathbb{R}^{K \times K}$ is a diagonal matrix where each diagonal entry satisfies
\begin{align*}
\Sigma_{x}(k,k)= Var(\kappa(x,U)|V=k)
\end{align*}
for $k=1,\ldots,K$.
\end{theorem}
\begin{proof}
To simplify the notations, all subsequent steps are implicitly conditioned on $X=x$. Based on Equation~\ref{eq1}, the conditional expectation is as follows:
\begin{align*}
E(Z_k) & = E\left( \frac{\sum_{j=1}^{m} A_j 1(V_j=k)}{m_k}\right) \\
& =  \frac{\sum_{j=1}^{m} E(A_j|V_j=k) 1(V_j=k)}{m_k} \\
& = E(\kappa(x,U_j)|V_j=k) \\
& = E(\kappa(x,U)|V=k),
\end{align*}
where the second line follows as $\mathcal{\vec{V}}$ is assumed to be fixed, the third line inserts $X=x$ into the expression, and the last line simplifies the expectation because $(U_j, V_j)$ are independently and identically distributed as $F_{UV}$.

Next, the conditional second moment is:
\begin{align*}
E(Z_k^2) =& E\left( \frac{(\sum_{j=1}^{m} A_j 1(V_j=k))^2}{m_k^2}\right) \\
=&\frac{E(\sum_{j=1}^{m} A_j^2 1(V_j=k))}{m_k^2} \\
&+\frac{E(\sum_{i=1}^{m}\sum_{j=1, j \neq i}^{m}A_i A_j 1(V_i=k) 1(V_j=k))}{m_k^2} \\
=&\frac{E(A_j^2|V_j=k)}{m_k} + (1-\frac{1}{m_k})E(A_i A_j|V_i=V_j=k)\\
=& \frac{E(A_j^2|V_j=k)}{m_k} + (1-\frac{1}{m_k})E^2(A_j|V_j=K).
\end{align*}
The last line follows because when conditioning on $X=x$, $A_i$ and $A_j$ are independent. The conditional variance can be computed as:
\begin{align*}
Var(Z_k) &=E(Z_k^2) - E^2(Z_k) \\
&=\frac{E(A_j^2|V_j=k)-E^2(A_j|V_j=k)}{m_k}\\
&=\frac{Var(A_j|V_j=k)}{m_k} \\
&=\frac{Var(\kappa(x,U)|V=k)}{m_k}. 
\end{align*}
Furthermore, for every $k \neq l$:
\begin{align*}
Cov(Z_k, Z_l) =&E(Z_k Z_l) - E(Z_k)E(Z_l) \\
=&\frac{E(\sum_{j=1}^{m} A_j 1(V_j=k)\sum_{j=1}^{m} A_j 1(V_j=l))}{m_k m_l} \\
&- \frac{E(\sum_{j=1}^{m} A_j 1(V_j=k))E(\sum_{j=1}^{m} A_j 1(V_j=l))}{m_k m_l}\\
=&0.
\end{align*}
This is because, when conditioning on $X=x$ and with $\mathcal{\vec{V}}$ being fixed, $A_j 1(V_j=k)$ only depends on $U_j$ and is always independent of $A_j 1(V_j=l)$.

As a result:
\begin{align*}
E(Z|(X=x))&=\mu_{x}\\
Var(Diag(\vec{m})^{0.5} \circ Z |X=x))&=\Sigma_{x}.
\end{align*}

Since $m_k=\Theta(m)=\Theta(n)$ and $Var(\kappa(x,U)|V=k)$ is bounded due to the finite-moment assumption, the variance converges to $0$ as $n$ increases. By Chebyshev's inequality, we immediately have:
\begin{align*}
Z|(X=x) \stackrel{n\rightarrow \infty}{\rightarrow} \mu_{x}.
\end{align*}

To prove the central limit theorem, we check the Lyapunov condition for $A_j$ per-dimension, and because the variance and the third moments are all bounded, we have:
\begin{align*}
s_{m_k}^2 = \sum_{j=1}^{m} Var(A_j|V_j=k)&=O(m_k), \\
\sum_{j=1}^{m} 1(V_j=k) E(|A_j- E(A_j)|^3)&=O(m_k).
\end{align*}
It follows that:
\begin{align*}
\frac{1}{s_{m_k}^{3}}\sum_{j=1}^{m} 1(V_j=k) E(|A_j- E(A_j)|^3)=O(\frac{1}{\sqrt{m_k}}) \rightarrow 0,
\end{align*}
so the Lyapunov condition is satisfied.

By the Lyapunov central limit theorem, for each dimension $k$ we have
\begin{align*}
\sqrt{m_k}(Z_{k}|(X=x) - \mu_{x}(k)) \stackrel{dist}{\rightarrow} \mathcal{N}(0,\Sigma_{x}(k,k)). 
\end{align*}
Concatenating every dimension yields
\begin{align*}
Diag(\vec{m})^{0.5} \circ (Z|(X=x) - \mu_{x}) \stackrel{dist}{\rightarrow} \mathcal{N}(0,\Sigma_{x}).
\end{align*}
\end{proof}
To summarize the proof: by conditioning on the class labels and $X=x$, each dimension $Z_k$ becomes an independent sum of i.i.d. random variables, allowing the central limit theorem to apply to each $Z_k$. Furthermore, the covariance between $Z_k$ and $Z_l$ is $0$ due to the disjointness of the class label sets, leading to the application of the multivariate central limit theorem.

However, when $x$ is different, the normal distribution is different. We aim to determine a condition under which the normal distribution is shared per-class, meaning that all vertices within the same class have the same normal distribution, even when $x$ varies.

\begin{theorem}
\label{thm2}
Given class $y$, suppose we have
\begin{align}
\label{eqCondition}
Var(E(\kappa(X,U)|V=k,Y=y))=0
\end{align}
for all of $k=1,\ldots,K$, Then, the encoder embedding for all observations within class $y$ converges to the same normal distribution. Specifically, as $m$ increases to infinity,
\begin{align*}
&Z|(Y=y) \stackrel{prob}{\rightarrow} \mu_{y},
\end{align*}
where $\mu_{y} \in \mathbb{R}^{K}$ satisfies
\begin{align*}
\mu_{y}(k)= E(\kappa(X,U)|V=k,Y=y)
\end{align*}
for each $k=1,\ldots,K$, and
\begin{align*}
Diag(\vec{m})^{0.5} \circ (Z|(Y=y) - \mu_{y}) \stackrel{dist}{\rightarrow} \mathcal{N}(0, \Sigma_{y})
\end{align*}
where $\Sigma_{y} \in \mathbb{R}^{K \times K}$ is a diagonal matrix, and each diagonal entry satisfies:
\begin{align*}
\Sigma_{y}(k,k)= E\{Var(\kappa(X,U)|V=k,Y=y)\}
\end{align*}
for each $k=1,\ldots,K$.
\end{theorem}
\begin{proof}
From Theorem~\ref{thm1}, we have
\begin{align*}
E(Z_k|X=x) & = E(\kappa(x,U)|V=k) \\
Var(Z_k | X=x) &=\frac{Var(\kappa(x,U)|V=k)}{m_k}.
\end{align*}
By performing conditional probability manipulation to condition on $Y$ but not $X$,  we obtain:
\begin{align*}
E(Z_k|Y=y) & = E(\kappa(X,U)|V=k,Y=y) \\
Var(Z_k|Y=y) &=\frac{E(Var(\kappa(X,U)|V=k,Y=y))}{m_k} \\
&+Var(E(\kappa(X,U)|V=k,Y=y)).
\end{align*}
Hence, when $Var(E(\kappa(X,U)|V=k,Y=y))=0$, we have $Var(Z_k|Y=y) =O(1/m_k)$. Similarly, when checking the covariance yields that $Cov(Z_k,Z_l|Y=y) =0$ for all $k \neq l$. As a result, the law of large number and central limit theorem in Theorem~\ref{thm1} extend to this case using the same arguments. 
\end{proof}

\subsection{Discriminant Analysis}

Given per-class normality, discriminant analysis is the Bayes optimal classifier \citep{DevroyeGyorfiLugosiBook}, leading to the following theorem.

\begin{theorem}
Suppose Equation~\ref{eqCondition} holds for every class $y$, and the class-conditional density of the encoder embedding is always bounded. Then, discriminant analysis is asymptotically Bayes optimal for the classification of $(Z,Y)$. 

Specifically, given $A \sim \text{Graph}(m,X, \kappa) \in \mathbb{R}^{m}$, and $Z$ as the encoder embedding of $A$, let $g_{m}(\cdot)$ denote the discriminant analysis classifier on $Z$, and $g^{*}_{m}(\cdot)$ denote the Bayes optimal classifier on $Z$. Given any argument $z$ in the support of $Z$, and assuming without loss of generality that $g^{*}_{m}(z)$ always permits a unique solution, it follows that
\begin{align*}
g^{*}_{m}(z) - g_{m}(z) \stackrel{m \rightarrow \infty}{\rightarrow} 0.
\end{align*}
\end{theorem}
\begin{proof}
Given any argument $z$ in the support of $Z$, the Bayes optimal classifier is defined as the class that maximizes the conditional probability:
\begin{align*}
g^{*}_{m}(z) = \arg\max_{y}Prob(Y=y|Z=z).
\end{align*}
Here,
\begin{align*}
Prob(Y=y|Z=z) = \frac{f_{y}^{m}(z) \pi_{y}}{\sum_{k=1}^{K} f_{k}^{m}(z) \pi_{k}},
\end{align*}
where $f_{y}^{m}(z)$ is the class-conditional density of the encoder embedding at graph size $m$. Since the denominator is constant across all possible $y$, it can be safely omitted, leading to:
\begin{align*}
g^{*}_{m}(z) = \arg\max_{y} (f_{y}^{m}(z) \pi_{y}).
\end{align*}

Discriminant analysis has the same formulation but, instead of using the true class-conditional density $f_{y}^{m}(z)$, which is usually unknown for sample data, it assumes that all class conditional densities are normal. It then estimates the mean and variance from the sample data. Let $h_y^{m}(z)$ denote the normal density used at graph size $m$. The discriminant analysis classifier is given by
\begin{align*}
g_{m}(z) = \arg\max_{y} (h_y^{m}(z) \pi_{y}).
\end{align*}

Therefore, $g^{*}_{m}(z)$ does not necessarily equals $g_{m}(z)$, because the former uses the true density, while the latter uses an approximated normal density. We aim to prove these two classifiers are asymptotically the same, i.e., 
\begin{align*}
g^{*}_{m}(z) - g_{m}(z) \stackrel{m \rightarrow \infty}{\rightarrow} 0.
\end{align*}

To prove this, consider Theorem~\ref{thm2}, which shows that the cumulative distribution of $Z|Y$ is asymptotically normal. Given that the density of the encoder embedding is assumed to be bounded, and since $h_y^{m}$ is also bounded as a normal density, with the sample mean and variance estimated under discriminant analysis converging to their true values, we have for each $y$,
\begin{align*}
f_{y}^{m}(z) - h_y^{m}(z) \rightarrow 0.
\end{align*}
In other words, as the ground-truth density $f_y^{m}(z)$ converges to a normal density, it must coincide with the estimated normal density $h_y^{m}(z)$ used in discriminant analysis.

Since $\pi_k$ is fixed, it follows that
\begin{align*}
(f_{y}^{m}(z) - h_y^{m}(z)) \pi_{y} \rightarrow 0
\end{align*}
for each $y$. Consequently,
\begin{align*}
g^{*}_{m}(z) - g_{m}(z) = \arg\max_{y} (f_{y}^{m}(z)\pi_{y}) - \arg\max_{y} (h_y^{m}(z) \pi_{y}) \rightarrow 0.
\end{align*}
It is important to note that while pointwise convergence does not generally imply convergence of the argmax, it does hold in this case. This is because the argmax operation here is applied to a finite and discrete set of $y$ values, and we assume a unique solution for the Bayes optimal classifier with a fixed $z$.

Note that it suffices to assume a unique solution to prove the equivalence. If there are multiple solutions, such that at a given $z$, there are several $y$ that achieves the argmax for $g^{*}_{m}(z)$ as $m$ goes to infinity, it does not matter which $y$ that $g_{m}(z)$ converges to. The discriminant analysis remains asymptotically Bayes optimal by achieving any of the argmax values. Moreover, the classification error remains the same, as long as the maximum conditional density is the same, regardless of how many $y$ values may achieve the argmax.
\end{proof}

In theory, quadratic discriminant analysis, which assumes each class-conditional density has a different covariance matrix, is the most suitable classifier. However, in practice, linear discriminant analysis, which simplifies by assuming class-conditional covariance matrices are all the same, often strikes a better balance between bias and variance. This makes it more numerically stable due to fewer parameters to estimate. Consequently, in the numerical sections, we consistently use linear discriminant analysis.

We should note that Equation~\ref{eqCondition} may not always be satisfied per-class. It holds, for instance, in cases like the stochastic block model where $\kappa(X,U)$ remains constant within a class. Additionally, it holds approximately when the underlying variable $X$ has a small standard deviation within each class, resulting in a value close to $0$ for $Var(E(\kappa(X,U)|V=k, Y=y))$.

\section{Experiments}

\subsection{Simulations}

In this section, we present two representative settings to illustrate the usage and properties of encoder embedding.

The first scenario involves a high-dimensional Gaussian distribution with the following parameters: $p=100$, $K=3$, and $Y$ taking on the values $\{1,2,3\}$ with equal probability. Within each class $k=\{1,2,3\}$, we set $X_k | (Y=k) \sim N(3,1)$, and $X_d | (Y=k) \sim N(1,0.5)$ for $d \neq k$. This setup ensures that the first three dimensions create class separation, while all other dimensions are identical in density across classes. For a given sample size $n$, we generate the sample data $\mathbf{X}$ and corresponding labels $\mathbf{Y}$ based on this multivariate Gaussian setting. We construct the general graph $\mathbf{A}$ by first calculating pairwise Euclidean distances and then applying a distance-to-kernel transformation \citep{DcorKernel}. Subsequently, we employ the sample encoder embedding, which results in a 3-dimensional representation.

The left panels in Figure~\ref{fig1} display the original $\mathbf{X}$ in the top three dimensions and the encoder embedding in 3D for a sample size of $n=1000$. The last panel on the left side of Figure~\ref{fig1} presents the 5-fold cross-validation error using linear discriminant analysis as $n$ increases from $50$ to $500$. For comparison, we also provide the 5-nearest-neighbor classification error using the encoder embedding, as well as the results when applying linear discriminant analysis (LDA) and 5-nearest-neighbor (5NN) classification directly to the original $\mathbf{X}$. This analysis is repeated for $100$ Monte-Carlo replicates, and the standard deviations are all within $1\%$.

The second setting considers a weighted graph generated from a stochastic block model. We still have $K=3$, and the labels $Y$ are distributed as $\{1,2,3\}$ with equal probabilities. The block probability matrix is defined as:
\begin{align*}
\mathbf{B}=
  \left[\begin{matrix}
    0.2 & 0.1 & 0.1 \\
    0.1 & 0.2 & 0.1 \\
    0.1 & 0.1 & 0.2 
  \end{matrix}\right]
\end{align*}
Then we let 
\begin{align*}
\kappa(x,u)= Q_{xu} * \operatorname{Bernoulli}(B(y, v))
\end{align*}
be the pairwise function, where $Q_{xu}$ is a random weight uniformly distributed in $[0,10]$. Then, for a given sample size $n$, we generate the weighted graph $\mathbf{A}$ and corresponding labels $\mathbf{Y}$ based on the given distribution, and apply the sample method to compute the embedding. The right panels in Figure~\ref{fig1} visualize the adjacency matrix $\mathbf{A}$ in heatmap, followed by the encoder embedding in 3D at $n=1000$. The last panel on the right in Figure~\ref{fig1} reports the 5-fold error using linear discriminant analysis as $n$ increases from $50$ to $500$. For comparison, we also report the 5-nearest-neighbor error using the encoder embedding, as well as applying linear discriminant analysis and 5-nearest-neighbor to the adjacency spectral embedding (ASE) of $\mathbf{A}$ into $d=30$. This process is repeated for $100$ different Monte-Carlo replicates, and all standard deviations are within $1\%$.

From Figure~\ref{fig1}, it is evident that the encoder embedding successfully separates the three classes of data and even outperforms the original data in terms of separability. This improved separability translates to a better classification error, as demonstrated in the bottom panels. On the bottom left panel, the encoder embedding achieves better finite-sample performance than using the original Euclidean data, which is noteworthy because LDA applied to the original data is also asymptotically optimal. On the bottom right panel, the encoder embedding outperforms the spectral embedding method, which are proven to be asymptotically optimal in binary graph classification \citep{SussmanEtAl2012, Priebe2019}.

Given the encoder embedding, both LDA and 5NN classifiers perform very well, with LDA being marginally better as the sample size $n$ increases. This indicates that the encoder embedding provides a strong foundation for classification tasks. Note that in both cases, because $K=3$ and each class is equally likely, the random-chance error is $2/3$.

\begin{figure}[htbp]
	\centering
	\includegraphics[width=0.85\textwidth,trim={0cm 0cm 0cm 0cm},clip]{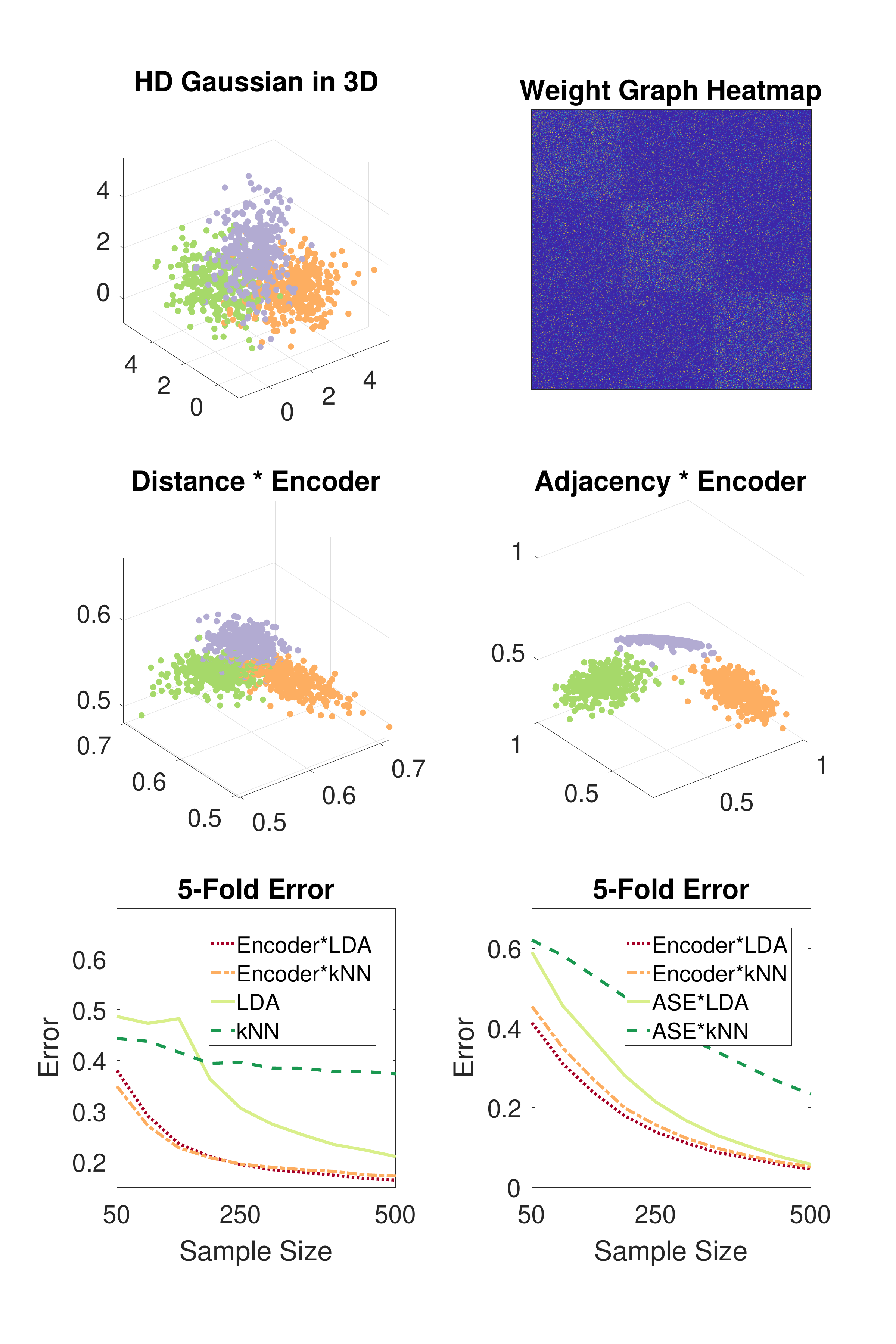}
	\caption{\normalsize This figure provides visualizations of the original data, the embedded data, and the resulting 5-fold classification error for both multivariate Gaussian data (on the left) and a weighted stochastic block model (on the right). In the embedding visualization, different colors represent observations from different classes.}
	\label{fig1}
\end{figure}

As another simulation example, we consider the exact same high-dimensional Gaussian distribution, except with $K=4$, and visualize three different metric choices followed by encoder embedding. From Figure~\ref{fig2}, it is evident that the encoder embedding performs effectively for any metric choice, whether it is based on Euclidean distance, Spearman rank correlation \citep{KendallBook}, or inner product, and all groups exhibit relative separability. The same holds true for higher values of $K$, although visualizing the separation in 3D becomes challenging. Real data experiments considered data with much larger $K$, demonstrating the method's effectiveness for such scenarios as well. In addition, it is worth noting that while this simulation is straightforward and most distance metrics work equally well, in practice, certain metrics may outperform others for real data with complex structures. 

\begin{figure}[htbp]
	\centering
	\includegraphics[width=0.85\textwidth,trim={0cm 0cm 0cm 0cm},clip]{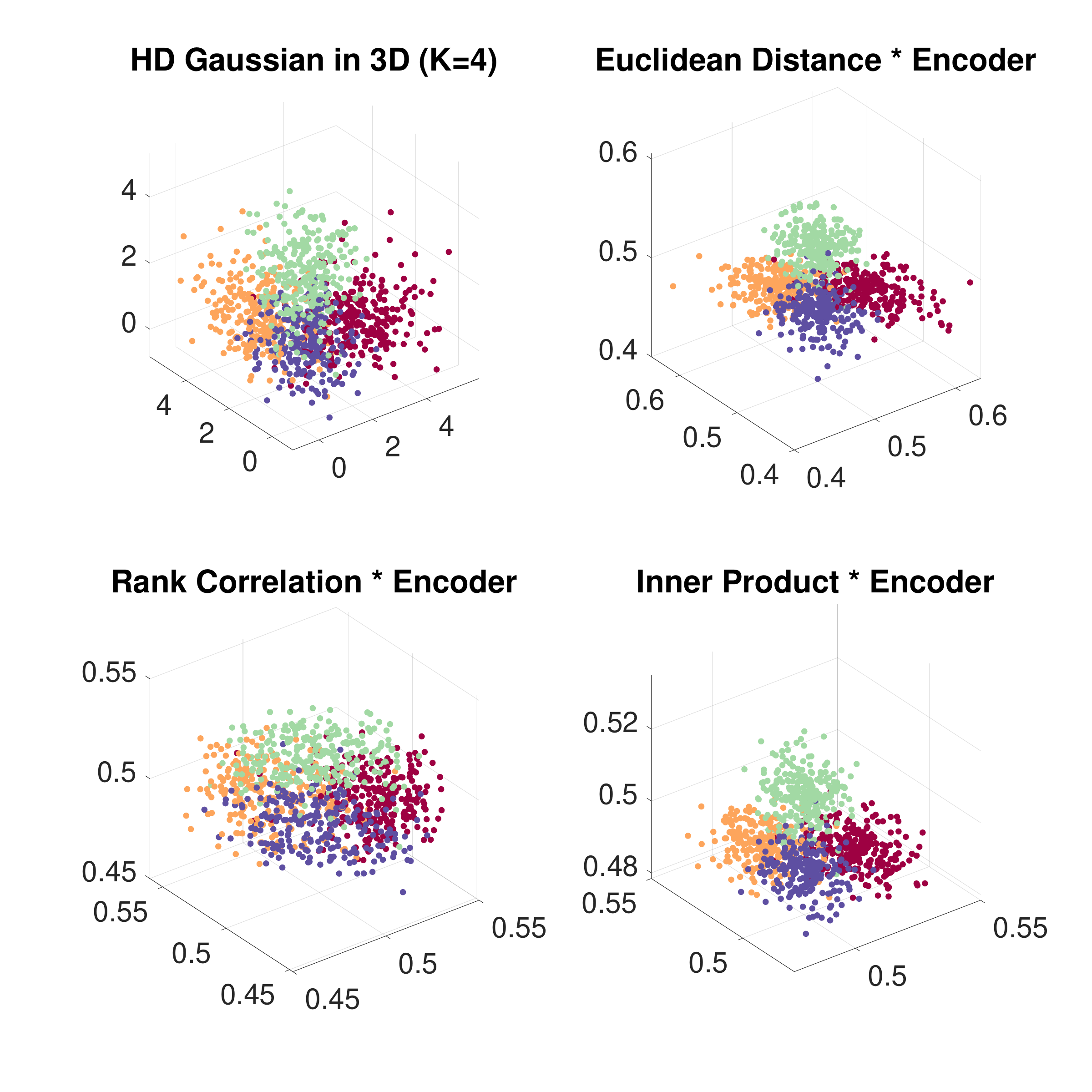}
	\caption{\normalsize This figure visualizes the original data and the embedded data using three different graph transformations.}
	\label{fig2}
\end{figure}

\subsection{Real Data}

In this section, we present the classification performance of the encoder embedding using a wide range of real-world data. The dataset includes three text datasets, two face image datasets, and two weighted graphs. Here are the details of each dataset:
\begin{itemize}
\item The Cora dataset \citep{mccallum2000automating} is a citation network with $2708$ papers and $7$ classes. Each paper has a $1433$-dimensional binary attribute indicating the absence or presence of corresponding words from a dictionary.
\item The Citeseer data \citep{giles1998citeseer} is another citation network with $3312$ papers and $6$ classes, with an attribute of $3703$ dimensions for words presence. 
\item ISOLET spoken data language \citep{cole1990spoken,chapelle2008optimization} is a database of spoken English letters, consisting of $7797$ spoken letters and $617$ dimensions. 
%\item Columbia Object Library (COIL-20) \citep{NeneNayarMurase1996a} is a database of gray-scale images of $20$ objects. It has $1440$ images, each of size $32 \times 32$.  
\item The Extended Yale B database \citep{GeorphiadesBelheumeurKriegman2001, LeeHoKriegman2005} contains $2414$ face images of $38$ individuals under various poses and lighting conditions. The images are resized to $32x32$ pixels.
%\item The ORL Database of Faces \citep{SamariaHarter1994} contains 400 images from 40 distinct subjects, each of size $32 \times 32$ pixels.
\item CMU PIE face images \citep{SimBakerBsat2003,HeEtAl2005} consists of $11554$ images from $K=68$ persons. 
\item The Wikipedia dataset \citep{GCCAJMVA,ManifoldPRL} includes two weighted graph representations, based on English text features and French text features. It comprises $1382$ Wikipedia articles and $5$ disjoint classes.
\end{itemize}

For each text and image dataset, except for the Wikipedia dataset that are already weighted graphs, we transform them into a general graph using Spearman rank correlation. This choice proves to be more effective than using Euclidean distance. The use of correlation or cosine distance is well-established in text analysis \citep{BleiNgJordan2003}, and the Spearman version is a rank-based approach that is more robust. This approach is also suitable for face images, as they are known to lie on different subspaces per class \citep{LeeHoKriegman2005, CaiEtAl2007}, making correlation or cosine-based metric a better candidate for these data types.

For each dataset and each method, we conducted 5-fold cross-validation, repeating the process $50$ times. The results, including the average classification error and standard deviation for the encoder embedding followed by linear discriminant analysis (LDA) or 5-nearest-neighbor (5NN), are summarized in Table 1. We also report three benchmarks: linear discriminant analysis (LDA), 5-nearest-neighbor (5NN), and a two-layer neural network (two-layer NN) using MATLAB's fitcnet method\footnote{\url{https://www.mathworks.com/help/stats/fitcnet.html}}, with a neuron size of 100 and all parameters set to their default values. These three benchmarks are applied directly to the raw data without the graph transformation.

Despite the diversity of datasets, the encoder embedding, whether applied to a transformed general graph or a weighted graph directly, consistently demonstrates excellent classification performance when coupled with linear discriminant analysis. It either achieves the best classification error or comes very close to the best error across all the real data experiments. While the encoder embedding with 5-nearest-neighbor also performs relatively well, there are cases where it falls significantly behind in terms of classification performance. %It is important to note that the experiments are designed to demonstrate that the encoder embedding is well-behaved in terms of classification error, rather than to assert that it is always the best performer. For instance, the performance of the neural network could certainly be improved with a more optimized architecture and better parameter selection beyond the default settings.

\begin{table}[ht]
\renewcommand{\arraystretch}{1.3}
\small
\centering
{\begin{tabular}{|c||c|c|c|c|c|}
 \hline
Error & Encoder*LDA & Encoder*5NN & LDA & 5NN & Two-Layer NN \\
\hline
Cora &31.5$\pm$0.3\%  & 30.6$\pm$0.5\% & 40.2$\pm$0.7\% &54.9$\pm$0.6\% &\textbf{26.5}$\pm$0.5\%\\
Citeseer &  \textbf{29.7}$\pm$0.1\%  & 31.9$\pm$0.4\% &61.4$\pm$1.4\% &61.1$\pm$0.8\% &32.3$\pm$0.6\%\\
Isolet & 8.4$\pm$0.1\%  & 19.1$\pm$0.2\% &5.7$\pm$0.1\% &11.4$\pm$0.2\% &\textbf{5.4}$\pm$0.3\%\\
\hline
Yale B & \textbf{1.2}$\pm$0.1\% & 27.4$\pm$0.8\%   &3.6$\pm$0.2\%  &34.0$\pm$0.8\% &88.6$\pm$3.0\% \\
PIE & 5.6$\pm$0.1\% & 24.5$\pm$0.3\% & \textbf{5.2}$\pm$0.1\% &15.8$\pm$0.3\% &94.3$\pm$1.3\%\\
\hline
Wiki TE & 19.3$\pm$0.4\%  & \textbf{17.9}$\pm$0.4\%  &25.4$\pm$0.3\% &22.1$\pm$0.5\% &19.4$\pm$0.9\%\\
Wiki TF &\textbf{17.7}$\pm$0.4\% & 18.0$\pm$0.4\% & 23.0$\pm$0.3\% &22.7$\pm$0.7\% &18.8$\pm$0.8\%\\
\hline
\hline
Time & Encoder*LDA & Encoder*5NN & LDA & 5NN & Two-Layer NN \\
\hline
Cora & 0.01 $\pm$ 0.00  & 0.01 $\pm$ 0.00 & 1.9 $\pm$ 0.1 & 0.32 $\pm$ 0.02 & 0.87 $\pm$ 0.05\\
Citeseer & 0.02 $\pm$ 0.01  & 0.02 $\pm$ 0.01 & 13.7 $\pm$ 0.5 & 1.6 $\pm$ 0.08 & 3.2 $\pm$ 0.4\\
Isolet & 0.12 $\pm$ 0.01  & 0.08 $\pm$ 0.01 & 1.5 $\pm$ 0.3 & 1.1 $\pm$ 0.1 & 5.7 $\pm$ 1.2\\
\hline
Yale B & 0.12 $\pm$ 0.03  & 0.03 $\pm$ 0.05 & 2.4 $\pm$ 1.1 & 0.2 $\pm$ 0.02 & 23.2 $\pm$ 6.4\\
PIE & 0.64 $\pm$ 0.04  & 0.30 $\pm$ 0.03 & 7.0 $\pm$ 3.0 & 3.8 $\pm$ 0.9 & 63.8 $\pm$ 21.6\\
\hline
Wiki TE & 0.01 $\pm$ 0.00  & 0.01 $\pm$ 0.00 & 0.19 $\pm$ 0.1 & 0.18 $\pm$ 0.1 & 2.2 $\pm$ 0.8\\
Wiki TF & 0.01 $\pm$ 0.00  & 0.01 $\pm$ 0.00 & 0.19 $\pm$ 0.1 & 0.18 $\pm$ 0.1 & 2.1 $\pm$ 0.6\\
\hline
\end{tabular}
\caption{\normalsize The table presents the classification errors and running times using the encoder embedding, in comparison to standard classifiers. The average results and one standard deviation are reported after $50$ random replicates. For each dataset, the best error is highlighted in bold. The running time only considers the method itself, excluding any input processing, such as converting raw data to a distance matrix.}
}
\end{table}

\section{Conclusion}
This paper presented the asymptotic theoretical properties of the one-hot graph encoder embedding when applied to general graph models, including binary graphs, weighted graphs, distance matrices, and kernel matrices. The experimental section demonstrated the superior performance of the encoder embedding for supervised learning across a diverse range of general graph datasets. 

It is important to note that our benchmarks are relatively straightforward, and there exist many other embedding techniques, classifiers, and advanced neural network architectures that could achieve better performance. Thus, the results of our real data experiments do not imply that encoder embedding is always the optimal choice. Rather, it serves as a viable technique for embedding general graphs, demonstrating its ability to preserve information well during the embedding process. As with any dimension reduction technique, it is possible that some information loss occurs during the embedding, which may lead to slight under-performance in certain cases. This was evident, for example, in the PIE face image data and Isolet data, which warrants further investigations into the distribution of the high-dimensional data.

There are many other areas for future work, such as the selection of distance metrics in general encoder embedding, the incorporation of multiple general graphs, the exploration of high-dimensional geometric implications, and the development of faster computation methods without explicit graph transformation. These directions could lead to further methodological improvements and deepen our understanding of general graph structures and their relationship to high-dimensional data.

\section*{Declarations}

\begin{itemize}
\item Availability of data and materials: The datasets generated and/or analysed during the current study are available in Github, [https://github.com/cshen6/GraphEmd].
\item Funding: This work was supported in part by the National Science Foundation DMS-2113099, and by funding from Microsoft Research.
\item Acknowledgements: We thank the editor and anonymous reviewers for providing valuable feedback that improved the exposition of the paper.
\end{itemize}

\bibliography{shen,general}

\end{document}